\documentclass[letterpaper, 10 pt, conference]{format/ieeeconf}  

\IEEEoverridecommandlockouts                              

\usepackage{dsfont}
\usepackage{marvosym}       

\usepackage{amsmath}

\usepackage{graphics} 
\usepackage{amsmath} 
\usepackage{amssymb}  
\usepackage{todonotes}

\usepackage{amsfonts}
\usepackage{mathtools}
\usepackage[linesnumbered,algoruled,boxed,vlined, noend]{algorithm2e}
\usepackage{amsthm}
\usepackage{comment}
\usepackage{cite}
\usepackage{wrapfig}
\let\labelindent\relax
\usepackage{enumitem}
\usepackage{overpic}
\usepackage{xspace}
\usepackage{multirow}

\DeclareMathAlphabet{\mathcal}{OMS}{cmsy}{m}{n} 

\newtheorem{problem}{Problem}[section]
\newtheorem{proposition}{Proposition}[section]

\newtheorem{theorem}{Theorem}[section]
\theoremstyle{definition}

\theoremstyle{remark}

\DeclareMathOperator*{\argmin}{argmin}

\DeclarePairedDelimiterX{\norm}[1]{\lVert}{\rVert}{#1}

{\end{list}}

\usepackage{soul} 
\usepackage[a-2b,mathxmp]{pdfx}[2018/12/22]

\def\prob{MBLR\xspace}
\def\dpalgo{DP-MBLR\xspace}
\def\mctsalgo{MCTS-MBLR\xspace}
\def\exact{OPT-MBLR\xspace}

\newif\ifarxiv
\arxivtrue

\title{\LARGE \bf 
On the Utility of  Buffers in Pick-n-Swap Based Lattice Rearrangement}
\author{Kai Gao \hspace{25mm} Jingjin Yu
\thanks{K. Gao and J. Yu are with the Department of 
Computer Science, Rutgers, the State University of New Jersey, Piscataway, NJ, USA. 
Emails: {\tt\small \{kai.gao, jingjin.yu\}@rutgers.edu}.
This work is partly supported by NSF awards IIS-1845888 and IIS-2132972, and an Amazon Research Award. }}
\begin{document}

\maketitle


\begin{abstract}
We investigate the utility of employing multiple buffers in solving a class of rearrangement problems with pick-n-swap manipulation primitives. In this problem, objects stored randomly in a lattice are to be sorted using a robot arm with $k\ge 1$ swap spaces or buffers, capable of holding up to $k$ objects on its end-effector simultaneously. 
On the structural side, we show that the addition of each new buffer brings diminishing returns in saving the end-effector travel distance while holding the total number of pick-n-swap operations at a minimum. 
This is due to an interesting recursive cycle structure in random $m$-permutation, where the largest cycle covers over $60\%$ of objects.
On the algorithmic side, we propose fast algorithms for 1D and 2D lattice rearrangement problems that can effectively use multiple buffers to boost solution optimality.
Numerical experiments demonstrate the efficiency and scalability of our methods, as well as confirm the diminishing return structure as more buffers are employed.

\vspace{2mm}
\noindent Introduction video: 
\url{https://youtu.be/KtBxoARGaVQ}
\end{abstract}

\section{Introduction}\label{sec:intro}

Efficient multi-object rearrangement is an essential skill for robots to master. The quality of a rearrangement solution largely depends on how often the end-effector picks up objects and, to a lesser extent, on the end-effector travel. 
Whereas most existing robot end-effectors can only pick and place a single object at a time, end-effectors capable of holding multiple objects have begun to pop up (see, e.g., Fig.~\ref{fig:multiHeadGripper}, top row). 
For these new end-effector types, new algorithms must be developed to exploit their capabilities fully. Toward this goal, the \emph{pick-n-swap} primitive, which allows an end-effector to hold one object and make object swaps temporarily, is systematically studied in \cite{yurearrangement} for objects stored in regular, lattice-like storage (see, e.g., Fig.~\ref{fig:multiHeadGripper}, bottom row). Exploring unique properties of the pick-n-swap problem on lattices, the study developed efficient algorithms for computing optimal or near-optimal solutions for minimizing the number of pick-n-swaps and end-effector travel.

In this work, we study the pick-n-swap problem further. As it is shown~\cite{yurearrangement} that holding one object and performing object swaps can significantly enhance the efficiency in solving rearrangement tasks (as compared with pick-n-place without swapping capability), a natural follow-up question is whether an additional gain is possible by making the end-effector even more capable. 
\begin{figure}[ht]
    \centering
    \includegraphics[width=\columnwidth]{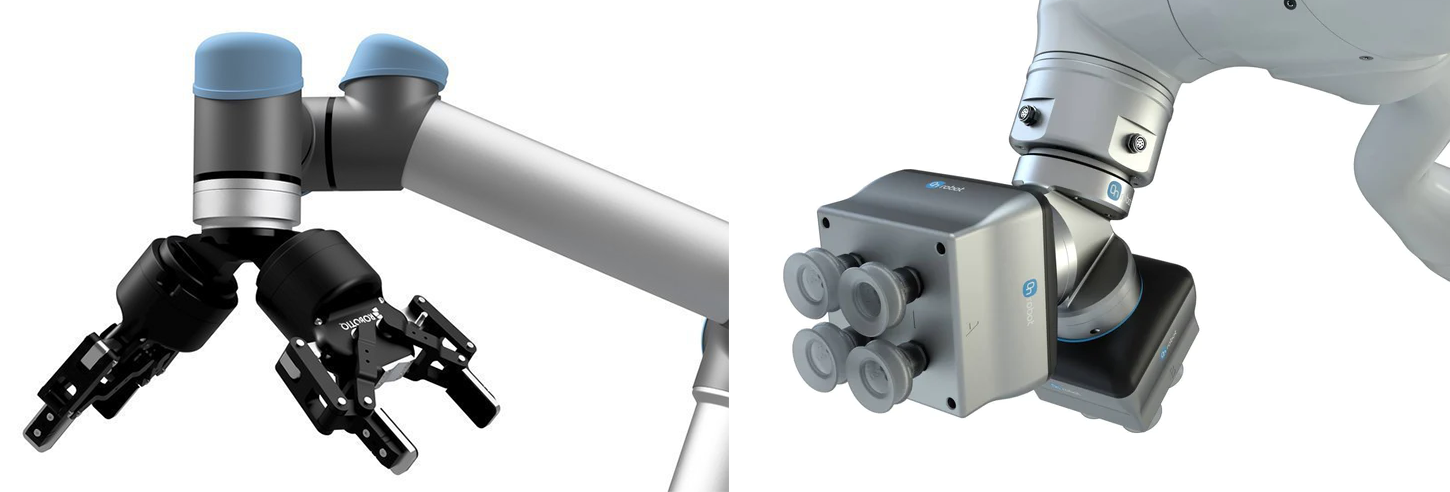}
    \vspace{-1mm} \\
    \includegraphics[width=\columnwidth]{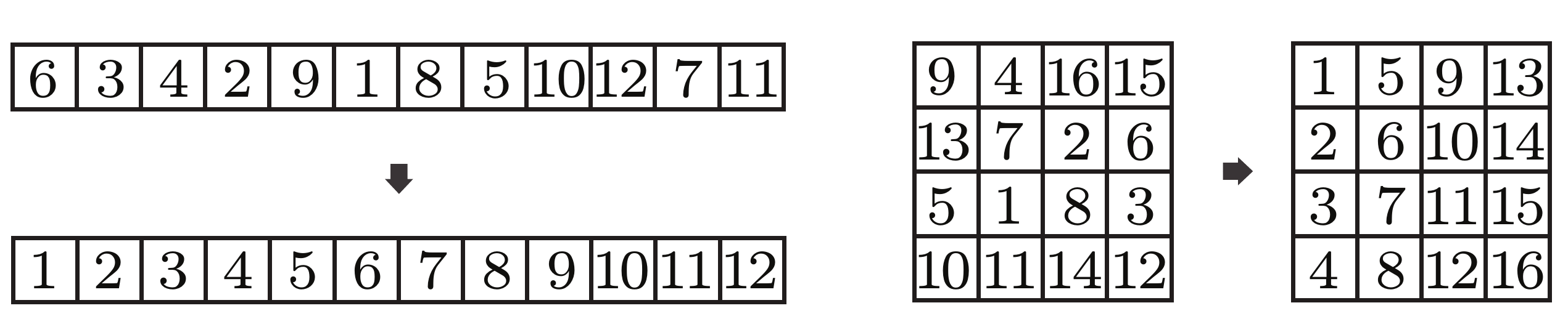}
    \vspace{-2mm}
\caption{[top left] A robot from Universal Robots equipped with dual Robotiq grippers. [top right] A dual OnRobot vacuum gripper setup. [bottom] Illustrations of 1D and 2D lattice rearrangement problems.
}
    \label{fig:multiHeadGripper}
\end{figure}
Specifically, we want to understand the benefit, if any, of allowing the end-effector to hold multiple objects at a time. We call each unit of capacity the end-effector has in temporarily holding an object as a \emph{buffer}, i.e., an end-effector with $k$ buffers can temporarily hold $k$ objects. Intuitively, having more buffers will allow more flexibility in rearrangement and thus reduce task execution time.

On the structural side, as it turns out, for rearranging randomly distributed objects stored in lattice-like storage, when the pick-n-swap time dominates the total cost, each additional buffer after the first provides increasingly diminished returns in improving the overall efficiency. Due to the observation, we can conclude that there is little to gain after having $2$-$3$ buffers, regardless of the total number of objects.  
On the algorithmic side, building on a \emph{cycle-following} algorithm proposed in \cite{yurearrangement} for a single buffer, we have developed effective and fast algorithms using dynamic programming (DP) for handling an arbitrary number of buffers for 1D lattices (i.e., objects form a single row), which is subsequently extended to 2D settings. We have also developed search-based algorithms for handling the general objective, allowing arbitrary cost trade-offs between the pick-n-swaps and end-effector travel. 

More concretely, the contributions of this work are:
\begin{itemize}[leftmargin=3.5mm]
\item We establish a tight lower bound of the expected length of the largest cycle in a random permutation. 
The result is supported by a simple yet elegant proof and has the potential to be of broader interest beyond the immediate scope of this work.

\item For the multi-buffer pick-n-swap-based rearrangement problem on lattices, for sequentially minimizing the number of pick-n-swaps and end-effector travel, applying the aforementioned structural result, we establish that the utility of having $k$ buffers exponentially decays with $k$. Quantitatively, having $3$ buffers is sufficient to reap over $90\%$ of the possible efficiency gain, implying that a more complex end-effector design may be unnecessary. 
\item We propose an effective dynamic programming (DP) method for solving multi-buffer pick-n-swap problems. Our method computes high-quality solutions for 1D and 2D lattices. Extensive numerical evaluations not only confirm our method's effectiveness but also confirm that adding more buffers provides a diminished return in efficiency gain. Our DP methods may be readily extended to higher dimensions.
\item We further develop an algorithm based on Monte Carlo tree search that supports arbitrary trade-offs between pick-n-swap costs and end-effector travel costs, providing a general solution to the multi-buffer pick-n-swap problem. 
\end{itemize}

\noindent\textbf{Organization}. The rest of the paper is organized as follows. 
In Sec.~\ref{sec:related}, we introduce previous works in various related problem domains.
In Sec.~\ref{sec:prob}, we provide a formal definition of the problem we study and review a crucial cycle structure in single buffer lattice rearrangement.
Then, in Sec.~\ref{sec:properties}, we establish some structural
properties on parallelizing rearrangement tasks with multiple buffers. We describe
our proposed algorithmic solutions in Sec.~\ref{sec:method}. The evaluation
follows in Sec.~\ref{sec:experiments}. We conclude with Sec.~\ref{sec:conclusion}.

\section{Related Works}\label{sec:related}
Autonomous robotic object rearrangement apply to both household settings\cite{wada2022reorientbot, gao2021fast, zeng2021transporter, danielczuk2021object} and logistics \cite{szegedy2020rearrangement, wang2021uniform, han2021toward, wang2021efficient, wang2022lazy}.
While rearrangement in shelf-like scenarios \cite{wada2022reorientbot, lee2019efficient, wang2021efficient, wang2022lazy, vieira2022persistent} focus more on robot-object collision avoidance, the ones in the tabletop\cite{danielczuk2021object, huang2019large, han2018complexity, shome2018fast, han2021toward, huang2021visual, gao2021fast, song2020multi} or conveyor scenario\cite{carlisle1994pivoting, han2019toward} tend to use overhand grasps and emphasize manipulation efficiency.
Similar to the tabletop setting, a recent work\cite{yurearrangement} studies rearrangement in lattices while allowing a pick-n-swap operation instead of the more standard pick-n-place operation.
Rearrangement with non-prehensile actions (e.g. pushes)\cite{huang2019large, huang2021dipn, han2021toward, huang2021visual, song2020multi} does not require accurate contact between the robot and objects but the resulting arrangements tend to be less predictable.
In contrast, rearrangement with prehensile actions (e.g. pick-n-place)\cite{wada2022reorientbot, lee2019efficient,huang2019large, krontiris2015dealing,krontiris2016efficiently, wang2021efficient,han2018complexity,labbe2020monte,wang2021uniform,gao2021fast,wang2022lazy} requires accurate grasping poses but more predictable manipulation enables long-horizon planning, leading to increased efficiency. 

Based on different scenarios and applications, the objective of rearrangement varies, with an initial focus on seeking feasible solutions in difficult and general rearrangement problems \cite{krontiris2015dealing, krontiris2016efficiently}. Gradually, the focus shifts to minimizing the number of manipulations to increase the system throughput \cite{wang2021uniform,wang2021efficient,huang2021visual,gao2021fast,wang2022lazy}. Additionally, some works\cite{han2018complexity} take the traveling distance of the gripper into consideration, minimizing the total execution time.
In this paper, we investigate lattice rearrangement with pick-n-swaps. Taking both the number of operations and the traveling distance of the end-effector into account, we seek to reduce the execution time of the rearrangement task.

In many research domains including robotic manipulation and operations research, capacity constraints frequently arise that add significant complexity. In rearrangement problems, a recent work \cite{gaorunning} studies the \emph{running buffer size}, which is the size of the needed free space for temporary object displacement in the rearrangement task. 
In the field of printed circuit board assembly, manipulation with multi-head grippers\cite{grunow2004operations, moghaddam2016parallelism, luo2014milp} enables the robot arm to carry multiple objects simultaneously and saves traveling time of the arm. Most of these works tend to assume that the picking locations or the placing positions are located close to each other. 
In our setting, locations of picking and placing are more entangled, making the problem challenging to solve.
\vspace{-1mm}

\section{Preliminaries}\label{sec:prob}
\vspace{-2mm}
\subsection{Problem Formulation}
Consider a $d$-dimensional lattice with size $m_1 \times m_2 \times \dots \times m_d$, with each cell storing exactly one object. 
The storage yields an arrangement $\mathcal A$ of the labeled objects.
A robot arm is tasked to move objects from a start arrangement $\mathcal A_s$ to a desired goal arrangement $\mathcal A_g$ with pick-n-swap operations.
Each pick-n-swap can be represented as a 3-tuple $(\ell,i,j)$, where the robot moves to the picking (lattice) position $\ell$, swap a held object $i$ with the object $j$ stored in the cell.
It is possible that $i=\varepsilon$ or $j=\varepsilon$, i.e., empty.
In this paper, we work with a $k$-buffer arm that can carry at most $k$ objects when traveling over the lattice.
A rearrangement plan $P=\{p_0, p_1, p_2, \dots p_N, p_{N+1} \}$ is a sequence of pick-n-swaps moving objects from $\mathcal A_s$ to $\mathcal A_g$ with 
$p_0=p_{N+1}=(r,\varepsilon, \varepsilon)$, where $r$ is the rest position of the robot arm. 
Correspondingly, the robot arm starts from a rest position $r$ in the lattice, executes pick-n-swaps one by one in $P$, and moves back to $r$ after finishing the task.
We assume the rest position in the lattice is $(1,\ldots,1)$.

Similar to \cite{yurearrangement}, we evaluate the quality of the rearrangement plan based on the execution time.
Assume that the time cost of one pick-n-swap is $c_p$ and that spent by the end-effector to travel a unit distance (the distance between adjacent lattice cells) is $c_t$,
then the cost function can be represented as:
\begin{equation}\label{eq:obj}
    J_T(P) = c_pN+ c_t \sum_{i=0}^{N} dist(p_i.\ell, p_{i+1}.\ell)
\end{equation}
where $dist(p_i.\ell, p_{i+1}.\ell)$ measures the Euclidean distance between two successive pick-n-swap positions in the lattice.
In practical scenarios, $c_p$ tends to be much larger than $c_t$ as object picking and placing involves prehensile manipulation. Therefore, while our paper does examine the general objective, the discussion focuses primarily on solutions that first minimize the number of pick-n-swaps.

Based on the definitions above, the problem investigated in this paper can be formulated as follows.
\begin{problem}[Multi-Buffer Lattice Rearrangement (\prob)]
Given the current arrangement $\mathcal A_s$, and the goal arrangement $\mathcal A_g$, compute a rearrangement plan $P$ for a $k$-buffer robot arm minimizing $J_T(P)$.
\end{problem}

Specifically, we denote \prob in $d$-dimensional lattice as $d$-\prob. It is assumed that $\mathcal A_s$ is a random arrangement of the objects and $\mathcal A_g$ is an ordered arrangement.

\subsection{Cycle Structures of Rearrangement on Lattices}
We briefly describe the key insights and methods from \cite{yurearrangement} for solving the single-buffer \prob. 
Since $\mathcal A_s$ is a random arrangement and $\mathcal A_g$ is an ordered arrangement, an \prob instance induces a random permutation $\pi$ of $m$ elements. The permutation contains one or more \emph{cycles}, each of which can be represented as a sequence of objects $(o_1, o_2, ..., o_k)$. 
For each object $o_i, 1\leq i \leq k-1$, its goal position is occupied by $o_{i+1}$, and the goal position of $o_k$ is occupied by $o_1$.

The top left figure (ignoring the arrow) in Fig.~\ref{fig:simpleex}  shows a simple instance of 1-\prob, which contains $2$ cycles: $(41)$ and $(53)$. As shown in \cite{yurearrangement}, the number of pick-n-swaps can be minimized by \textbf{\emph{cycle-following}}, which is to perform pick-n-swap operations following each cycle sequentially before moving to the next cycle. To rearrange each cycle, the minimum number of pick-n-swap operations is the size of the cycle plus $1$. For the instance shown in Fig.~\ref{fig:simpleex}, the cycle $(41)$ is first solved, followed by solving the cycle $(53)$. This yields the plan $(1, \varepsilon, \varepsilon)$, $(1, \varepsilon, 4)$, $(4, 4, 1)$, $(1, 1, \varepsilon)$, $(3, \varepsilon, 5)$, $(5, 5, 3)$, $(3, 3, \varepsilon)$, $(1, \varepsilon,\varepsilon)$, with $6$ pick-n-swaps and a total end-effector travel distance of $14$. The plan is illustrated in the first row of Fig.~\ref{fig:simpleex}. 
\begin{figure}[ht]
    \centering
    \includegraphics[width=\columnwidth]{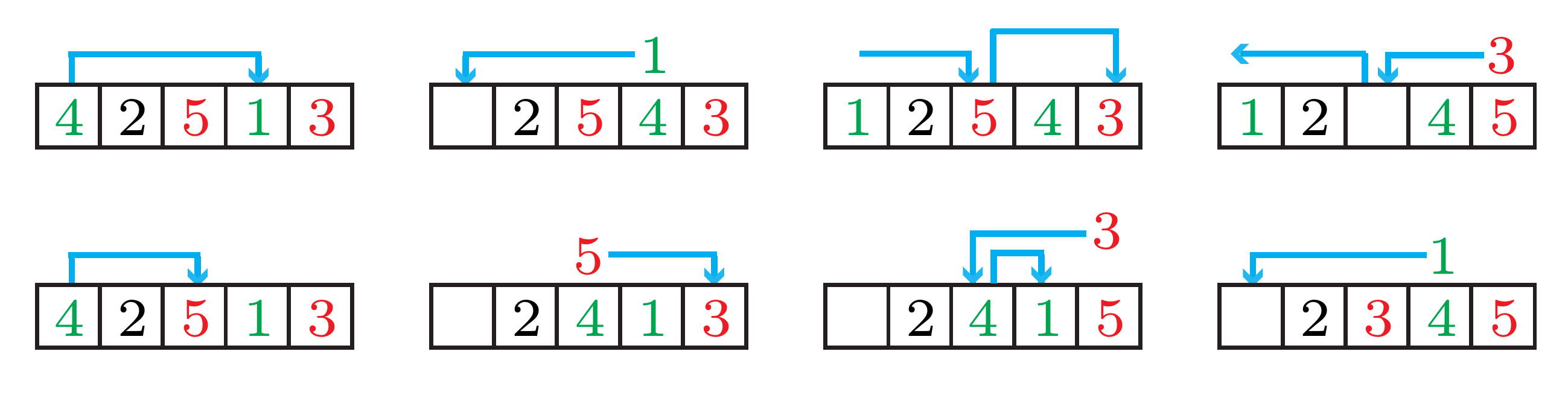}
    \caption{An simple example of 1-\prob with 5 elements, with two cycles $(41)$ and $(53)$ (top left). The two rows show two solution plans.}
    \label{fig:simpleex}
\end{figure}

While cycle-following minimizes the number of pick-n-swaps, end-effector travel distance can be further improved without adding pick-n-swaps via \textbf{\emph{cycle-switching}}. The idea of cycle-switching is that, as the end-effector passes over a new cycle while following another cycle, switching to following the new cycle can reduce the end-effector travel distance. For the example in Fig.~\ref{fig:simpleex}, the cycle $(41)$ is first followed. As the end-effector passes over object $5$ while holding object $4$, it switches to solving the cycle $(53)$, after which it finishes the cycle $(41)$. 
The plan is illustrated in the second row of Fig.~\ref{fig:simpleex} with $6$ pick-n-swaps and a total end-effector travel distance of $10$; we omit the corresponding plan sequence. 

Fig.~\ref{fig:workingExample} shows a slightly more complex example with three cycles and the optimal single-buffer pick-n-swap plan, which has $11$ pick-n-swaps and a travel distance of 16. We will refer back to this instance later.  
\begin{figure}[ht]
    \centering
    \includegraphics[width=\columnwidth]{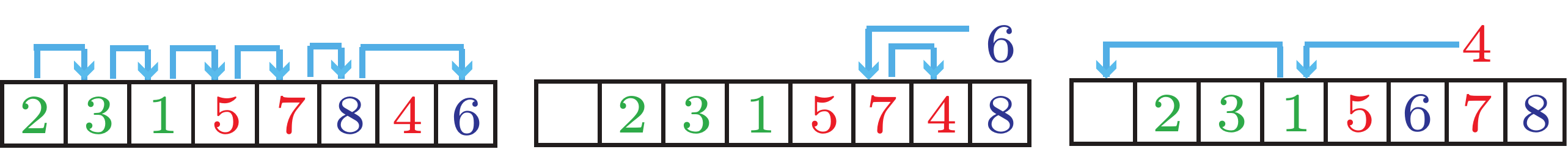}
    \caption{A slightly more complex instance of 1-\prob with three cycles $(231)$, $(574)$, and $(86)$ and the optimal single-buffer pick-n-swap based rearrangement plan.}
    \label{fig:workingExample}
\end{figure}

\section{Diminishing Return of Swap Buffers}\label{sec:properties}
\label{sec:properties}
With multiple buffers, a robot arm can carry more than one object simultaneously. As we will explain in more detail shortly, because each cycle of size $|c|$ requires at least $|c| + 1$ pick-n-swaps to rearrange, having multiple buffers will not help reduce the number of pick-n-swaps; the question is whether we can save on end-effector travel. This seems to be a difficult question because the number of buffers can be unlimited. Surprisingly, we show that for sequentially minimizing Eq.~\eqref{eq:obj} prioritizing the number of pick-n-swaps, more buffers only add exponentially diminishing improvements. We then discuss how multiple buffers may help reduce end-effector travel distance.

\subsection{Multi-Buffer Solutions Minimizing Pick-n-Swaps}

With a single buffer, solving each cycle of size $|c|$ requires $|c| + 1$ pick-n-swaps. It is easy to see that this does not change when there are more than one buffer. Furthermore, given a solution to a $k$-buffer \prob instance minimizing the total number of pick-n-swaps, it must be that the solution can be decomposed into $k$ disjoint partial solutions, one for each buffer, such that each partial solution is a cycle-switching solution for a single-buffer problem. Some of the $k$ partial solutions may be empty.  That is, 

\begin{proposition}\label{prop:cycle2buffer}
A $k$-buffer solution to an \prob instance minimizing the number of pick-n-swaps can be decomposed into $k$ partial solutions such that each partial solution is a cycle-switching solution for a disjoint subset of cycles of the \prob instance. This further implies that each cycle in the instance is handled by a single buffer. 
\end{proposition}

\begin{proof}
First of all, a solution with the minimum number of pick-n-swaps can be readily computed: we may simply use a single buffer to solve the problem and do not use any of the $k-1$ additional buffers. 
Now, if the statement does not hold, there are three  possibilities: (1) it decomposes into more than $k$ non-empty partial solutions, each over a disjoint set of cycles, (2) some cycle is handled by multiple buffers simultaneously, and (3) both of (1) and (2). If (2) or (3) is the case, then two objects of the same cycle $c$ are held in the buffer. This means that the cycle requires at least $|c| + 2$ pick-n-swaps to solve, which in turn means that the solution cannot minimize the number of pick-n-swaps. If instead (1) happens, we can readily merge some partial plans so that there are at most $k$ partial solutions, each covering a disjoint set of cycles. 
\end{proof}

\subsection{Object Distribution among Cycles}
Prop.~\ref{prop:cycle2buffer} says that each cycle is handled by a single buffer if we want to minimize the number of pick-n-swaps. If we can show that most objects are contained in only a small number of cycles, then we know that only a small number of buffers are needed. As it turns out, this is indeed the case. First, we look at the average size of the largest cycle in a random permutation.


\begin{theorem}[Golomb-Dickman Constant\cite{golomb1964random}]\label{t:golomb}
    Let $a_n$ denote the average size of the largest cycle among all permutations of size $n$, then
    \begin{equation}
        \lim_{n\rightarrow \infty} \dfrac{a_n}{n}=0.6243299885...
    \end{equation}
\end{theorem}

The proof of Theorem~\ref{t:golomb} is rather complex. In this work, using rudimentary analysis, we establish a lower bound fairly close to the Golomb-Dickman Constant.
\begin{theorem}\label{t:lcsize}
The expected size of the largest cycle in a random permutation of $m$ letters is lower bounded by 0.607$m$ asymptotically.
\end{theorem}
\ifarxiv
\begin{proof}
Let $P_m^i (1\leq i \leq m)$ denote the probability that the largest cycle of a random permutation on $m$ letters has size $i$. When $i > \lfloor m/2 \rfloor$, as shown in \cite{gal2003cell}, we have $\binom{m}{i} (i-1)!(m-i)!$ permutations whose largest cycle has size $i$. 

Therefore, for $i > \lfloor m/2 \rfloor$, 
\begin{align}
P_m^i = \frac{1}{m!}\binom{m}{i}(i-1)!(m-i)! = \frac{1}{i}
\end{align}

When $i\leq \lfloor m/2 \rfloor$, $P_m^i$ is lower bounded by the probability that a permutation has a unique largest cycle with size $i$.

\begin{equation}
\begin{split}
    P_m^i &\geq \dfrac{1}{m!}(
    \begin{pmatrix} m \\ i \end{pmatrix} (i-1)!(\sum_{j=1}^{i-1} P_{m-i}^j)(m-i)!
    )\\
    &= \dfrac{1}{i}\sum_{j=1}^{i-1} P_{m-i}^j = \dfrac{1}{i}(1-\sum_{j=i}^{m-i} P_{m-i}^j)\\
\end{split}
\end{equation}

Specifically, if $\lfloor m/3 \rfloor < i \le \lfloor m/2 \rfloor$, $i>(m-i)/2$. Therefore, for $\lfloor m/3 \rfloor < i \le \lfloor m/2 \rfloor$, 
\begin{equation}
    \begin{split}
    P_m^i &\geq \dfrac{1}{i}(1-\sum_{j=i}^{m-i} 1/j) \geq \dfrac{1}{i}(1-ln\dfrac{m-i}{i-1}). 
    \end{split}
\end{equation}

The expected size of the largest cycle is then 
\begin{equation}
\begin{split}
    \sum_{i=1}^m iP_m^i
    =&\sum_{i=\lfloor m/2 \rfloor+1}^{m} (iP_m^i) + \sum_{i=\lfloor m/3 \rfloor+1}^{\lfloor m/2 \rfloor} (iP_m^i) + \sum_{i=1}^{\lfloor m/3 \rfloor} (iP_m^i)\\
    \geq & \lceil m/2 \rceil + \sum_{i=\lfloor m/3 \rfloor+1}^{\lfloor m/2 \rfloor} (1-ln\dfrac{m-i}{i-1})\\
    \geq & \lceil 2m/3 \rceil -\sum_{\substack{\lfloor m/3 \rfloor+1 \leq a \leq b \leq \lfloor m/2 \rfloor\\ a+b=\lfloor m/3 \rfloor+1+\lfloor m/2 \rfloor}}ln\dfrac{(m-a)(m-b)}{(a-1)(b-1)}\\ 
\end{split}
\end{equation}

When $\lfloor m/3 \rfloor+1 \leq a\leq b \leq \lfloor m/2 \rfloor$ and $a+b=\lfloor m/3 \rfloor+1+\lfloor m/2 \rfloor$, we have $5m/6-1\leq a + b \leq 5m/6+1$ and $m^2/6-m/3 \leq ab \leq (5m/6+1)^2/4$. Therefore,

\begin{displaymath}
    \begin{split}
        &\lim_{m\rightarrow \infty}\sum_{i=1}^m iP_m^i\\
        \geq & \lfloor 2m/3 \rfloor -\sum_{\substack{\lfloor m/3 \rfloor+1 \leq a \leq b \leq \lfloor m/2 \rfloor\\ a+b=\lfloor m/3 \rfloor+1+\lfloor m/2 \rfloor}}ln\dfrac{(m-a)(m-b)}{(a-1)(b-1)}\\ 
        = & \lfloor 2m/3 \rfloor -\sum_{\substack{\lfloor m/3 \rfloor+1 \leq a \leq b \leq \lfloor m/2 \rfloor\\ a+b=\lfloor m/3 \rfloor+1+\lfloor m/2 \rfloor}}ln\dfrac{m^2-(a+b)m+ab}{ab-(a+b)+1}\\
        \geq & \lfloor 2m/3 \rfloor -(\dfrac{m/6}{2}+1)ln\dfrac{49m^2/144+17m/12+1/4}{m^2/6-7m/6}\\
        \geq & \lfloor 2m/3 \rfloor -(\dfrac{m/6}{2}+1) ln(49/24-\dfrac{125m-36}{24m(m-7)})\\
        \geq & (2m/3 -1) -(\dfrac{m/6}{2}+1) ln(49/24))\\
        \geq & 0.6071m-1.72 
    \end{split}
\end{displaymath}
\end{proof}
\fi

We note that Thm.~\ref{t:golomb} applies recursively since if we remove the largest cycle from a random permutation, the rest of the permutation is still random. This suggests that, for reasonably large $m$, the largest $k=3$ cycles occupies over $(1 - (1-0.624)^3) > 94.6\%$ of all elements.

Alternatively, this can also be observed by estimating the number of cycles in a random permutation of $m$ letters, which equals $H_m -1$ \cite{yurearrangement}, in which $H_m \approx \ln m$ is the $m$-th harmonic number. For two different numbers $m_1$ and $m_2$, $m_1 \ge m_2$, the number of cycles differs by $H_{m_1} - H_{m_2} \approx \ln \frac{m_1}{m_2}$. If we set $H_{m_1} - H_{m_2} = 3$, then $m_2 \approx 0.05 m_1$. In other words, three cycles will occupy about $95\%$ of elements, agreeing with the estimate using Thm.~\ref{t:golomb}. 

\subsection{Task Parallelism with Multiple Buffers}
The analysis so far shows the solution structure of \prob when the number of pick-n-swaps is minimized. In this subsection, we will discuss task parallelization based on \emph{cycle groups}, each of which is a group of cycles with range overlaps in 1-\prob\cite{yurearrangement}. In high-dimensional instances, all cycles belong to a single cycle group. When paralleling the rearrangement of objects $i$ and $j$, there are three different cases: (1) $i$ and $j$ are in the same cycle; (2) $i$ and $j$ are in different cycle groups; (3) $i$ and $j$ are in the same cycle group but in different cycles.

When rearranging objects $i$ and $j$, which are in the same permutation cycle, according to Prop.~\ref{prop:cycle2buffer}, if the rearrangement of $i$ and $j$ is parallelized, then the number of pick-n-swaps in the plan exceeds the minimum.


When objects $i$ and $j$ are in different cycle groups, task parallelization does not reduce the total distance since the range of cycles in different cycle groups does not overlap.
Additionally, for 1-\prob, we prove the independence of cycle groups in the rearrangement plan computation.

\begin{proposition}\label{prop:GroupParallel}
In 1-\prob, rearranging objects from different cycle groups can be considered separately.
\end{proposition}
\ifarxiv
\begin{proof}
Given a cycle group $G$, a rearrangement plan for $G$ is a plan $P_G$ that starts at the rest position, only rearranges objects in $G$, and finally returns to the rest position.

Suppose that there are $n$ cycle groups $G_1$, $G_2$, ..., $G_n$ ordered by their ranges from left to right. $P_1$, $P_2$, ..., $P_n$ are rearrangement plans(not necessarily optimal) for the cycle groups respectively.
Let the traveling distance in a plan $P$ be $dist(P)$. Since the rest position is at the first cell, the total distance of the rearrangement plan $P$ is lower bounded: 

\begin{equation}\label{eq:distanceBound}
    dist(P) \geq \sum_{i=1}^n dist(P_i) - \sum_{i=2}^n 2|max(G_{i-1})|
\end{equation}

In Algo.~\ref{alg:group}, we compute a rearrangement plan for the instance based on plans for each cycle group. 
Generally speaking, the robot arm starts rearrangement from the leftmost cycle group $G_1$, it moves to the next cycle group on the right when it pick-n-swaps at the rightmost cell of the current cycle group $max(G_i)$. 
When the arm finishes the rearrangement in the current cycle group, it resumes the rearrangement in the previous cycle group on the left.
In Lines 1-2, the algorithm records the indices of actions at $max(G_i)$ in each $P_i$. The indices indicate the moments when the robot arm moves to the next cycle group.
\begin{algorithm}[ht]
\begin{small}
    \SetKwInOut{Input}{Input}
    \SetKwInOut{Output}{Output}
    \SetKwComment{Comment}{\% }{}
    \caption{Cycle Group Switching}
		\label{alg:group}
    \SetAlgoLined
		\vspace{0.5mm}
		\Input{$P_1, P_2, ..., P_n$}
        \Output{$P$: The rearrangement plan for the instance}
		\vspace{0.5mm}
		\For{$1 \leq i \leq n-1$}
		{
		$I_i \leftarrow $ The index of the first pick-n-swap at the cell of $max(G_i)$ in $P_i$
		}
		$P\leftarrow (r,\epsilon,\epsilon) + P_1[2,...,I_1]+\dots+ P_{n-1}[2,\dots,I_{n-1}]+P_n[2,...,|P_n|-1]+P_{n-1}[I_{n-1}+1, ..., |P_2|-1] +\dots+P_{1}[(I_{1}+1), ..., |P_1|-1] + (r,\epsilon,\epsilon)$
\end{small}
\end{algorithm}
For the plan $P$ derived by Algo\ref{alg:group}, $dist(P)$ meets the lower bound in Eq.~\eqref{eq:distanceBound}. Therefore, with Algo.~\ref{alg:group}, minimizing $dist(P)$ is equivalent to minimizing each $dist(P_i)$ individually.
\end{proof}
\fi

Finally, for objects in the same cycle group but in different cycles, using the running example from Fig.~\ref{fig:workingExample}, we show parallel rearrangement using multiple buffers reduces the traveling cost without increasing the number of pick-n-swaps.
In the example, the cycles $(574)$ and $(86)$ can be handled in parallel.
If the robot arm has two buffers, it can place object $7$ on its way to cell $8$, and place object $6$ on its way back to cell $4$. In this way, the traveling distance is reduced by $2$ (total: $14 \to 12$). The new plan using two buffers is illustrated in Fig.~\ref{fig:workingExample-2}.

\begin{figure}[ht]
    \centering
    \includegraphics[width=0.725\columnwidth]{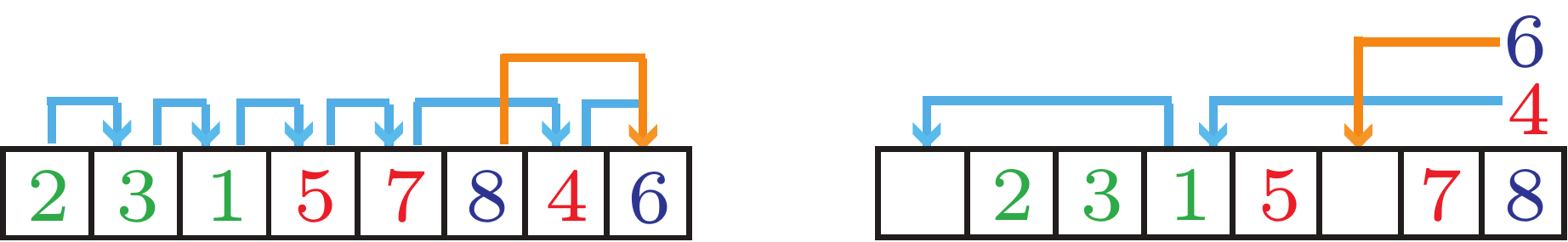}
    \caption{Using a second buffer (indicated by the orange arrows), parallelism can be realized by arranging $(574)$ and $(86)$ in the same cycle group.}
    \label{fig:workingExample-2}
\end{figure}

\section{Methods}\label{sec:method}
In this section, we propose fast algorithms for \prob.
Based on properties established in Sec.~\ref{sec:properties}, we introduce dynamic programming (DP) based algorithms for computing \prob solutions with the minimum number of pick-n-swaps that seek to minimize end-effector travel, for one- and two-dimensional lattices. Our methods also apply to higher dimensions. Additionally, for \prob with the general objective function \eqref{eq:obj}, we propose an algorithm based on Monte Carlo Tree Search (MCTS) for minimizing it.

\subsection{DP based Algorithm for 1-\prob Minimizing the number of Pick-n-Swaps (\dpalgo)}

Algo.~\ref{alg:oneDimension} shows the pipeline of the DP algorithm for 1-\prob. Since rearrangement plans for different cycle groups can be computed independently (Prop.~\ref{prop:GroupParallel}), without loss of optimality, we first decompose the instance into cycle groups to simplify the computation (Line 1). In Line 3, rearrangement tasks of objects in the same cycle will be assigned to the same buffer in order to minimize the number of pick-n-swaps (Prop.~\ref{prop:cycle2buffer}). Based on the assignment, we compute the optimal rearrangement plan minimizing the number of pick-n-swaps and traveling distance for each buffer using the exact algorithm for single-buffer \cite{yurearrangement} (Line 5). After that, we merge the task sequences to obtain a high-quality rearrangement plan for the cycle group. Finally, we concatenate plans for cycle groups with 
\ifarxiv
Algo.~\ref{alg:group} 
\fi
to minimize the additional cost when traveling among cycle groups.

\begin{algorithm}[ht]
    \SetKwInOut{Input}{Input}
    \SetKwInOut{Output}{Output}
    \SetKwComment{Comment}{\% }{}
    \caption{\dpalgo for 1-\prob}
		\label{alg:oneDimension}
    \SetAlgoLined
		\vspace{0.5mm}
		\Input{$\mathcal A_s$: Start arrangement in a 1-dimensional lattice arrangement.\\
		        $k$: the number of buffers
		}
        \Output{$P$: The rearrangement plan for the instance}
		\vspace{0.5mm}
		$G_1, G_2, ..., G_n \leftarrow $ CycleGroupDecomposition($\mathcal A_s$)\\
		\For{$ 1\leq i \leq n$}
		{
		$C_1, C_2, ..., C_k \leftarrow$ CycleAssignment($G$)\\
		\For{$1 \leq j \leq k$}{
		$s_j\leftarrow$ SingleBufferRearrangement($C_j$)
		}
		$P_i\leftarrow$ TaskSequencesMerge($\{s_1, ..., s_k\}$)
		}
		$P\leftarrow$ CycleGroupSwitching($\{P_1, ..., P_n\}$)
\end{algorithm}

\subsubsection{Task Assignment to Buffers}\label{sec:cycleAssignment}
In this step, we assign permutation cycles to $k$ buffers maximizing the minimum number of pick-n-swaps assigned to each buffer. This problem is equivalent to Bottleneck Multiple Subset Sum Problem (B-MSSP)\cite{caprara2000multiple}.
For example, consider an \prob instance with two buffers for four permutation cycles $c_1-c_4$, which contain $10, 5, 5, 2$ objects respectively. In this case, the needed pick-n-swaps for each cycle is $11,6,6,3$. The optimal task assignment to buffers is $(c_1,c_4), (c_2, c_3)$, with the minimum number of pick-n-swaps to be $12$. 
Given $k$ buffers and $c$ cycles, when $c \leq k$, we assign one cycle to each of the first $c$ buffers. Otherwise, we solve this problem with an integer linear programming (ILP) model\cite{caprara2000multiple}.

\subsubsection{Merging Task Sequences of Buffers}\label{sec:sequenceMerging}
For each cycle group, we propose a DP-based method for minimizing the total travel cost given the task sequences. Let $T$ be a $k$-tuple $(t_1, t_2, ..., t_k)$ representing a subproblem where only the first $t_i$ tasks are executed for buffer $i$. Let $f$ be the index of the buffer that the last executed task belongs to. For each combination of $T$ and $f$ $(T[f]\geq 1)$, $DP[T,f].cost$ counts the traveling distance from the moment when the arm leaves the rest position to the moment when the arm finishes the $T[f]^{th}$ task of buffer $f$. 
It is computed with the following equation.
\begin{equation*}
    \begin{split}
    & DP[T,f].cost = \\
    & \min_{\substack{
    T'=(t_1, ..., t_f-1, ..., t_k)\\
    T'[f']\geq 1
    }}(DP[T',f'].cost
    + dist((T',f'),(T,f)))
    \end{split}
\end{equation*}

where $(T',f')$($(T,f)$) represents the position of the $T'[f']^{th}$ task of buffer $f'$ ( the $T[f]^{th}$ task of buffer $f$).
Let $T_c$ be the $k$-tuple representing the original problem.
The minimum travel cost is derived based on the equation below:
\begin{equation*}
\min_{T_c[f]\geq 1} DP[T_c, f].cost+dist((T_c,f),r)
\end{equation*}
where $r$ is the rest position.
In a $k-$buffer scenario, the computation time of the dynamic programming approach is $O(m^k)$, which is polynomial in the number of objects $m$. 

For the running example from Fig.~\ref{fig:workingExample}, using $2$ buffers, \dpalgo first decomposes objects into cycles groups $G_1=\{(231)\}$ and $G_2=\{(574), (86)\}$.
For $G_1$, the robot arm moves objects to goal positions in the ordering of $2$, $3$, and $1$ with buffer $1$.
For $G_2$, \dpalgo assigns one cycle to each buffer: $C_1=\{(574)\}$, $C_2=\{(86)\}$. 
Note that according to Prop.~\ref{prop:GroupParallel}, even though buffer $1$ has claimed the tasks in $G_1$, 
it does not affect the total cost no matter which cycle in $G_2$ is assigned to buffer 1.
After \dpalgo computes a single buffer plan in $G_2$, 
it merges the rearrangement plans inside $G_2$ and between cycle groups as shown in Fig.~\ref{fig:sequence_merging}.

\begin{figure}[ht]
    \centering
    \vspace{1mm}
    \includegraphics[width=0.49\textwidth]{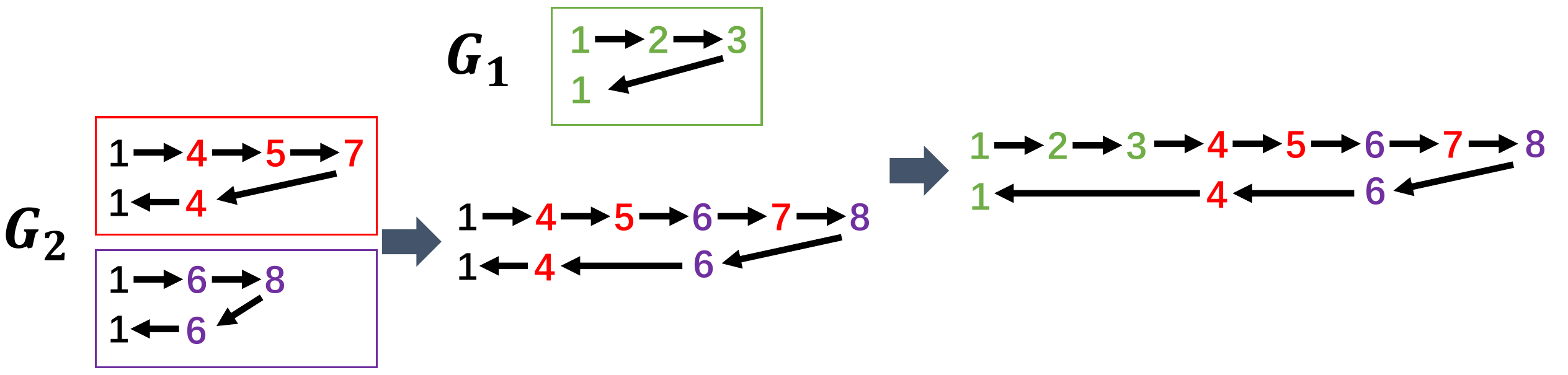}
    \caption{\dpalgo on the running example (Fig.~\ref{fig:workingExample}) with two buffers. This yields the plan shown in Fig.~\ref{fig:workingExample-2}, reducing traveling distance from $16$ to $14$.}
    \label{fig:sequence_merging}
\end{figure}

\subsection{Extension of \dpalgo to High Dimensions}
\dpalgo can be easily extended to $d$-\prob $(d>1)$ with a three-step pipeline. First, assign cycles to buffers with the ILP solver in Sec.~\ref{sec:cycleAssignment}. Second, for each buffer, we compute a plan with an asymptotically distance-optimal algorithm in the single buffer scenario\cite{yurearrangement}. Finally, we merge task sequences of buffers with dynamic programming as Sec.~\ref{sec:sequenceMerging}.

\subsection{Monte Carlo Tree Search for General Objective}
To decide on the next pick-n-swap in optimizing the general form of $J_T(P)$ in Eq.~\eqref{eq:obj}, our MCTS planner maintains a search tree where each node represents a rearrangement state containing three parts of information: object arrangement, objects in buffers, and the robot position. Each edge represents a pick-n-swap. For each node $s$, let $c(s)$ be the cost from the root node to $s$ with the same $c_p$ and $c_t$ in $J_T(P)$. 

For each state, we reduce the action set with two rules without loss of optimality:
\begin{itemize}[leftmargin=4mm]
    \item[a.] Only execute pick-n-swaps at cells $c$ where the target object $\mathcal A_g[c]$ is not placed inside.
    \item[b.] When executing a pick-n-swap at a cell $i$, if $i$ is in buffers, then place $i$.
\end{itemize}

Additionally, for 1-\prob, we notice that it is inefficient to bypass cell $i$ without placing object $i$ if it is in buffers. 
Based on the observation, the next pick-n-swap position is limited to the range between the closest goal positions on the left and right (if any) of objects in the buffers. Take the instance in Fig.~\ref{fig:workingExample} as an example, if the arm is located at $7$ and the objects in buffers are $1$ and $4$, then the range of the next pick-n-swap location is $[4,8]$.

In the selection stage of a node $s$, the MCTS planner chooses an action $a$ based on the following formula in the same spirit as the upper confidence bound (UCB):
\begin{equation}
    \argmin_a (\dfrac{w(f(s,a))}{n(f(s,a))}-C\sqrt{\dfrac{\log(n(s))}{n(f(s,a))}})
\end{equation}
where $f(s,a)$ is the child node of $s$ after action $a$, $w(s)$ and $n(s)$ are the total cost and times of visits at $s$.

\section{Evaluation}\label{sec:experiments}
In this section, we show experiments on the lattice rearrangement problem in two different scenarios, \prob with minimum pick-n-swaps and \prob with general cost functions, with a focus on the former.
The experiments are executed on an Intel$^\circledR$ Xeon$^\circledR$ CPU at 3.00GHz. 
Each data point is the average of 30 test cases except for unfinished trials, if any, given a time limit of 600 seconds for each test case.
Given a random $m$-permutation, the workspace dimension of 1-\prob and 2-\prob in the experiments are $1\times m$ and $\lceil\sqrt{m}\rceil\times\lceil\sqrt{m}\rceil$ respectively (see Fig.~\ref{fig:multiHeadGripper}).


\subsection{\prob with Minimum Number of Pick-n-Swaps}
For 1-\prob with the minimum number of pick-n-swaps, we first compare \dpalgo with a brute force exact algorithm \exact where we enumerate all possible cycle assignment and cycle-switching options to get the minimum distance cost.
Fig.~\ref{fig:LOROptimal}[left] shows the computation time.
When $m=20, 25$, \exact fails in some test cases. In these cases, the computation time is recorded as 600 seconds. 
Fig.~\ref{fig:LOROptimal}[right] shows the distance cost of computed solutions as proportions of the shortest distance in the single buffer setting.
We observe that $2$-$3$ buffers provide significant savings in end-effector travel. 
The results suggest that
compared with the exact algorithm for multiple buffers,
\dpalgo is much more scalable with little loss of optimality.

\begin{figure}[ht]
    \centering
    \includegraphics[width=0.49\textwidth]{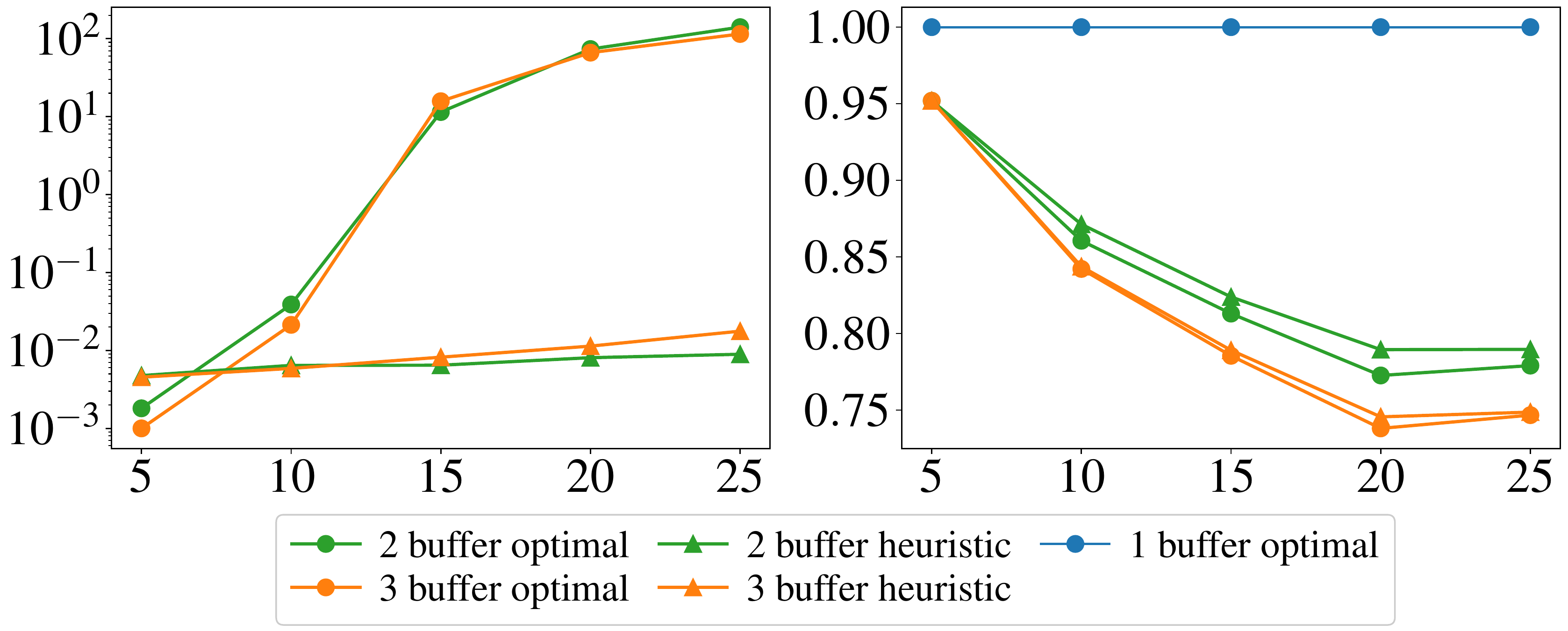}
    \caption{Comparison between \dpalgo and OPT-MBLR. [left] Computation time comparison. [right] Distance cost of computed solutions as proportions of the shortest distance in the single buffer setting.}
    \label{fig:LOROptimal}
\end{figure}

Fig.~\ref{fig:LOR} shows the performance of \dpalgo on 1-\prob with $2$-$5$ buffers.
In terms of computation time, the dominant component of \dpalgo is TaskSequenceMerge, 
which is polynomial in $m$ in the worst case. As a result, \dpalgo is highly scalable. 
Regarding optimality, compared with the single buffer setting, lattice rearrangement with two buffers saves travel distance by $25\%-30\%$. For $3$ buffers, the saving is about $30\%-35\%$. Adding more buffers beyond three, however, provides negligible additional benefits as predicted by our theoretical analysis.
That is, as the largest three cycles cover more than $93\%$ of objects and each cycle requires at most one unique buffer, additional buffers can barely make any difference.
\begin{figure}[ht]
    \centering
    \includegraphics[width=0.46\textwidth]{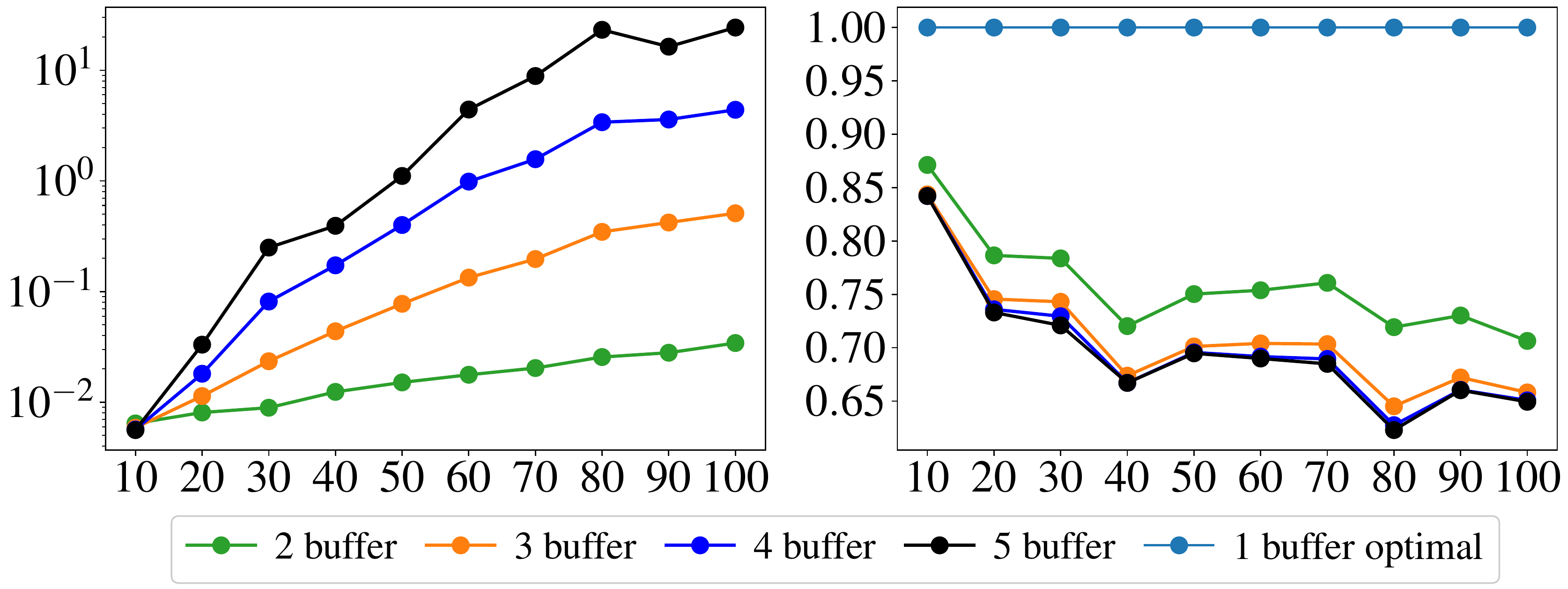}
    \caption{Performance of \dpalgo on 1-\prob instances with different numbers of buffers and objects. [left] Computation time (secs). [right] Distance cost as a proportion of that in optimal solutions with one buffer.}
    \label{fig:LOR}
\end{figure}

Fig.~\ref{fig:LTR} shows the performance of \dpalgo in 2-\prob.
Compared with single buffer rearrangement, 2-\prob solutions with two buffers save distance cost by $15\%-20\%$ and those with three or more buffers save distance cost by $20\%-25\%$.
While there is slightly less efficiency gain using multiple buffers in 2-\prob, the diminishing return from additional buffers is consistent with the results in 1-\prob.
The results show the efficiency of \dpalgo in 2-\prob and the effectiveness of robot arms with 2 or 3 buffers in two-dimensional lattices.

\begin{figure}[ht]
\vspace{1mm}
    \centering
    \includegraphics[width=0.46\textwidth]{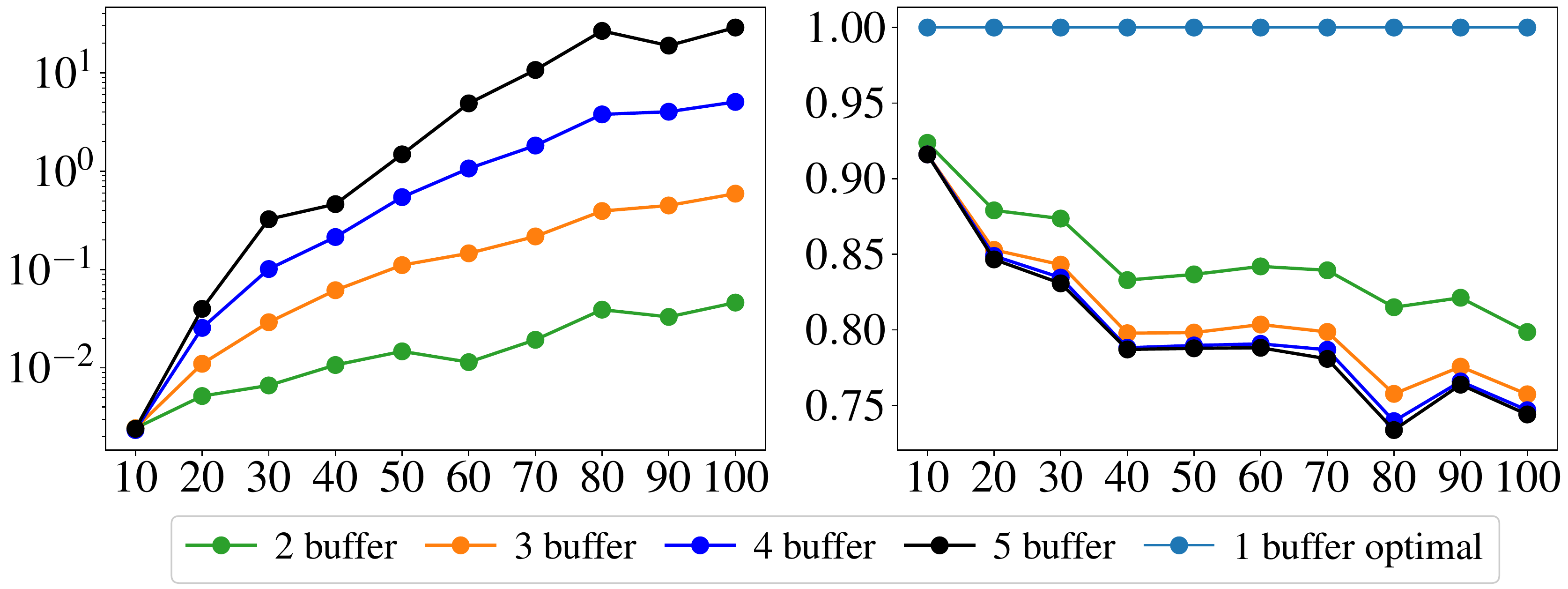}
    \caption{Performance of \dpalgo in 2-\prob instances with different number of objects. [left] Computation time (secs). [right] Distance cost as a proportion of that in optimal solutions with one buffer.}
    \label{fig:LTR}
\end{figure}
\subsection{\prob with General Objective Functions}
To evaluate the performance of \mctsalgo, we solve \prob with two different cost functions where the execution time of pick-n-swaps is not dominant in the rearrangement process.
When $c_p:c_t = 1:1$, both the number of pick-n-swaps and the traveling distance affect the cost in a computed solution.
When $c_p:c_t = 1:10^5$, the traveling cost is dominant.
Fig.~\ref{fig:mcts} shows the cost of solutions computed by \mctsalgo as a proportion of that in the optimal single-buffer solutions.
As the number of objects increases, rearrangement with multiple buffers has the potential to save more cost on traveling.
However, in the meanwhile, the computation time for each step of the decision decreases. 
Therefore, when $m=20$, less saving on traveling is seen by using multiple buffers.

\begin{figure}[ht]
\vspace{1mm}
    \centering
    \includegraphics[width=0.46\textwidth]{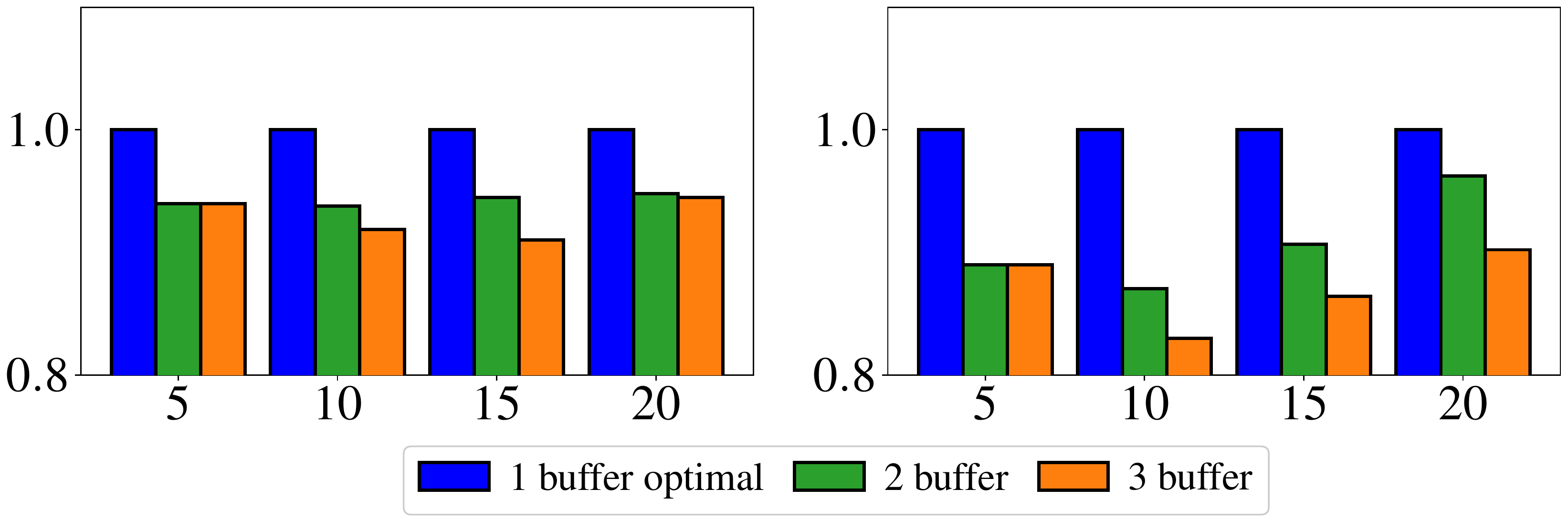}
    \caption{Cost of solutions computed by \mctsalgo as a proportion of that in optimal single-buffer solutions.[Left] $c_p:c_t = 1:1$, [Right] $c_p:c_t = 1:10^5$. }
    \label{fig:mcts}
\end{figure}


\section{Conclusion}\label{sec:conclusion}
In this work, we study the problem of employing multiple buffers to perform pick-n-swap-based rearrangement on lattices where the robot arm can hold multiple objects simultaneously when traveling.
In analyzing the effectiveness of additional swap spaces, we prove multiple structural results on parallelizing rearrangement tasks with multiple buffers.
Specifically, using the Golomb-Dickman constant, we establish an exponentially diminishing return in each additional buffer when the number of pick-n-swaps is minimized. We also provide a rudimentary proof of how a tight lower bound of the Golomb-Dickman constant can be obtained.
Using the obtained properties, we propose an efficient algorithm minimizing the number of pick-n-swaps based on dynamic programming.
For general cost functions, we also propose a Monte Carlo Tree Search-based method.
Experiments show the effectiveness of using multiple buffers compared with using one buffer in lattice rearrangement problems and further verify the properties we obtain.
\newpage
\bibliographystyle{format/IEEEtran}
\bibliography{bib/c}

\end{document}